\documentclass{llncs}

\usepackage[latin1]{inputenc}            
\usepackage[T1]{fontenc}              
\usepackage[dvips]{graphicx}
\usepackage[dvips]{color}
\usepackage{amsmath}
\usepackage{amssymb}
\usepackage{amsfonts}
\usepackage{epsfig}

\newcommand{\GDigit}[1]{\ensuremath{\mathcal{D}_{1}(#1)}}

\newcommand{\Eq}[1]{Eq.~(\ref{#1})}
\newcommand{\Conv}[1]{\ensuremath{\mathrm{conv}(#1)}}

\newcommand{\Equ}[1]{Eq.~(\ref{#1})}
\newcommand{\DC}{\ensuremath{C}}
\newcommand{\PT}[1]{\ensuremath{{\DC}_{#1}}}
\newcommand{\PTS}[2]{\ensuremath{[{\DC}_{#1}{\DC}_{#2}]}}
\newcommand{\SPRED}[2]{\ensuremath{S(#1,#2)}}
\newtheorem{thm}{Theorem}[section]
\newtheorem{hypo}[thm]{Hypothesis}
\newcommand{\VEC}[2]{\ensuremath{(#1, #2)}}

\newcommand{\RefFig}[1]{Fig.~\ref{#1}}
\newcommand{\RefDefinition}[1]{Definition~\ref{#1}}
\newcommand{\RefTheorem}[1]{Theorem~\ref{#1}}
\newcommand{\RefProposition}[1]{Proposition~\ref{#1}}
\newcommand{\RefLemma}[1]{Lemma~\ref{#1}}
\newcommand{\RefCorollary}[1]{Corollary~\ref{#1}}
\newcommand{\Lnorm}[1]{\ensuremath{\mathcal{L}^1(#1)}}
\begin{document}

  \title{Maximal digital straight segments and convergence of discrete
  geometric estimators}

  \author{\normalsize 
    François de Vieilleville, Jacques-Olivier Lachaud   \\
    \small 
    LaBRI, Univ. Bordeaux 1 \\
    351 cours de la Lib\'{e}ration, 33405 Talence, France.\\
    {\tt \{devieill,lachaud\}@labri.fr} \\~\\
    \normalsize Fabien Feschet\\
    \small LLAIC1, Clermont-Ferrand \\
    Campus des C{\'e}zeaux, 63172 Aubi{\`e}re Cedex, France\\
    {\tt feschet@llaic3.u-clermont1.fr}
  \normalsize}

  \maketitle

  \thispagestyle{empty}
  \begin{abstract}
    \noindent
    Discrete geometric estimators approach geometric quantities on
    digitized shapes without any knowledge of the continuous shape. A
    classical yet difficult problem is to show that an estimator
    asymptotically converges toward the true geometric quantity as the
    resolution increases. We study here the convergence of local
    estimators based on Digital Straight Segment (DSS) recognition. It
    is closely linked to the asymptotic growth of maximal DSS, for
    which we show bounds both about their number and sizes. These
    results not only give better insights about digitized curves but
    indicate that curvature estimators based on local DSS recognition
    are not likely to converge. We indeed invalidate an hypothesis
    which was essential in the only known convergence theorem of a
    discrete curvature estimator.
%%     We examine here the possible convergence a curvature estimator,
%%     which is one of the best available at coarse resolutions. Although
%%     a convergence theorem for this estimator was published, it was
%%     based on a hypothesis related to asymptotic properties of maximal
%%     Digital Straight Segments (DSS).  We show here that this
%%     hypothesis is asymptotically false. 
    The proof involves results from arithmetic properties of digital
    lines, digital convexity, combinatorics, continued fractions and
    random polytopes.

%%     It exploits a result on convex digital polygons, related to random
%%     polytopes.
%------------------------------------------------------------------------------
%%     The proof is based on a careful study of the links
%%     between convex digital polygons and maximal DSS. It also involved
%%     combinatorial representations of digital straight
%%     lines. Intermediate results provide properties of length and
%%     numbers of maximal DSS.
%------------------------------------------------------------------------------
%%     In this paper, links between convex digital polygon and maximal
%%     Digital Straight Segments (DSS) are studied. Properties arise in
%%     terms of bounds about length and numbers. Asymptotic convergence
%%     of digital estimators based on DSS is studied over real convex
%%     shapes with nice properties when digitization goes finer and
%%     finer. It is shown that the increase in number and length of DSS
%%     is not sufficient to guarantee an asymptotic convergence for the
%%     estimation of curvature. The existence of an convergent curvature
%%     estimator still remains an open problem.
%------------------------------------------------------------------------------
  \end{abstract}

  %------------------------------------------------------------------------------
  \section{Introduction}
  Estimating geometric features of shapes or curves solely on their
  digitization is a classical problem in image analysis and pattern
  recognition. Some of the geometric features are global: area,
  perimeter, moments. Others are local: tangents, normals,
  curvature. Algorithms that performs this task on digitized objects
  are called {\em discrete geometric estimators}. An
  interesting property these estimators should have is to converge
  towards the continuous geometric measure as the digitization
  resolution increases. However, few estimators have been proved to be
  convergent. In all works, shapes are generally supposed to have a
  smooth boundary (at least twice differentiable) and either to be
  convex or to have a finite number of inflexion points. The shape
  perimeter estimation has for instance been tackled in
  \cite{Kovalevsky92}. It proved the convergence of a perimeter
  estimator based on curve segmentation by maximal DSS. The speed of
  convergence of several length estimators has also been studied in
  \cite{Coeurjolly03}. Klette and \v{Z}uni\'{c} \cite{Klette00} survey results
  about the convergence (and the speed of convergence) of several
  global geometric estimators. They show that discrete
  moments converge toward continuous moments.
%%   As a consequence, the shape area, its
%%   center of gravity, its orientation and its elongation can be estimated more
%%   and more precisely as the resolution increases.
  
  As far as we know, there is only one work that deals with the
  convergence of local geometric estimators \cite{Coeurjolly02}. The
  symmetric tangent estimator appears to be convergent subject to an
  hypothesis on the growth of DSS as the resolution increases
  (see Hypothesis~\ref{hyp:coeurjolly}).
%%   The tangent
%%   direction is evaluated as the direction of the longest symmetric DSS
%%   recognized around the point of interest. 
  The same hypothesis entails that a curvature estimator is
  convergent: it is based on DSS recognition and circumscribed circle
  computation (see \RefDefinition{def:curvature-by-circumcircle}).
  
  In this paper, we relate the number and the lengths of DSS to the
  number and lengths of edges of convex hulls of digitized
  shapes. Using arguments related to digital convex polygons and a
  theorem induced by random polytopes theory \cite{Balog1991}, we
  estimate the asymptotic behaviour of both quantities. We
  theoretically show that maximal DSS do not follow the hypothesis
  used in \cite{Coeurjolly02}. Experiments confirm
  our result. The convergence theorem is thus not applicable to
  digital curves. As a consequence, the existence of convergent
  digital curvature estimators
%% based on DSS for the curvature computation
  remains an open problem. The paper is organized as follows. First,
  we recall some standard notions of digital geometry and combinatoric
  representation of digital lines, i.e. patterns. The relations
  between maximal segments and edges of convex digital polygons are
  then studied to get bounds on maximal segments lengths and
  number. Finally, the asymptotic behaviour of maximal segments is
  deduced from the asymptotic behaviour of convex digital
  polygons. The growth of some DSS is thus proved to be too slow to
  ensure the convergence of curvature estimation. This theoretical
  result is further confirmed by experiments.

%% Worse, they indicate that it is not even convergent.

  %------------------------------------------------------------------------------
  \section{Maximal digital straight segments}
  We restrict our study to the geometry of 4-connected digital
  curves. A digital object is a set of pixels and its boundary in
  $\mathbb{R}^2$ is a collection of vertices and edges. The boundary
  forms a 4-connected curve in the sense used in the present
  paper. Our work may easily be adapted to 8-connected curves. In the
  paper, all the reasoning are made in the first octant, but extends
  naturally to the whole digital plane. The digital curve is denoted
  by $\DC$. Its points $(\PT{k})$ are assumed to be indexed.  A set of
  successive points of $\DC$ ordered increasingly from index $i$ to
  $j$ will be conveniently denoted by $\PTS{i}{j}$ when no
  ambiguities are raised.

  \subsection{Standard line, digital straight segment, maximal segments}
  \begin{definition}(R\'eveill\`es \cite{Rev91})
    The set of points $(x,y)$ of the digital plane verifying $\mu \le
    ax-by < \mu + |a|+|b|$, with $a$, $b$ and $\mu$ integer numbers, is
    called the {\em standard line} with slope $a/b$ and shift $\mu$.
  \end{definition}
  The {\em standard lines} are the 4-connected discrete lines. The
  quantity $ax-by$ is called the {\em remainder} of the line. The
  points whose remainder is $\mu$ (resp. $\mu+|a|+|b|-1$) are called upper
  (resp. lower) leaning points. The principal upper and lower leaning
  points are defined as those with extremal $x$ values.  Finite
  connected portions of digital lines define \textit{digital straight
    segment}.  Since we work with restricted parts of $\DC$, we always
  suppose that indices are totally ordered on this part.

  \begin{definition} 
    A set of successive points $\PTS{i}{j}$ of $\DC$ is a {\em digital
      straight segment (DSS)} iff there exists a standard line $D(a,b,\mu)$
    containing them. The predicate ``$\PTS{i}{j}$ is a DSS'' is denoted by
    $\SPRED{i}{j}$. 
    %% When $\SPRED{i}{j}$ is true, we denote by $\DSL{i}{j}$ the
    %%     characteristics associated with the digital straight segment
    %%     \cite{DebRev94}: $(a,b,\mu)$, the end points $\PT{i}$ and
    %%     $\PT{j}$, the principal upper and lower leaning points $U_m$, $U_M$,
    %%     $L_m$, $L_M$.
  \end{definition}

  The first index $j$, $i \le j $, such that $\SPRED{i}{j}$ and $\neg
  \SPRED{i}{j+1}$ is called the {\em front} of $i$. The map associating any
  $i$ to its front is denoted by $F$. Symmetrically, the first index
  $i$ such that $\SPRED{i}{j}$ and $\neg \SPRED{i-1}{j}$ is called the {\em
    back} of $j$ and the corresponding mapping is denoted by $B$. 

  %% We get the following obvious relations.

  %% \begin{proposition}
  %% \label{prp:elementary-properties-of-DSS}
  %% \begin{itemize}

  %%   \item[(i)] $\forall i \le i' \le j' \le j, \SPRED{i}{j} \Implies \SPRED{i'}{j'}$;

  %%   \item[(ii)] $F$ and $B$ are locally increasing; 

  %%   \item[(iii)] $F \circ B \circ F = F \mathrm{~and~} B \circ F \circ B = B$.

  %% \end{itemize}

  %% \end{proposition}
  
  Maximal segments form the longest possible DSS in the curve. They
  are essential when analyzing digital curves: they provide tangent
  estimations \cite{Feschet99,Lachaud05a}, they are used for
  polygonizing the curve into the minimum number of segments
  \cite{Feschet03}.

  \begin{definition}
    \label{def:maximal-segment}
    Any set of points $\PTS{i}{j}$ is called a {\em maximal segment} iff
    any of the following equivalent characterizations holds: (1)
    $\SPRED{i}{j}$ and $\neg \SPRED{i}{j+1}$ and $\neg \SPRED{i-1}{j}$, (2)
    $B(j)=i$ and $F(i)=j$, (3) $\exists k, i=B(k)$ and $j=F(B(k))$, (4)
    $\exists k', i=B(F(k'))$ and $j=F(k')$.
  \end{definition}

  %% \begin{definition}
  %% Any set of points $\PTS{i}{j}$ is called a {\em maximal
  %% segment} iff $\SPRED{i}{j}$ and $\neg \SPRED{i}{j+1}$ and $\neg
  %% \SPRED{i-1}{j}$. An equivalent characterization is $B(j)=i$ and $F(i)=j$.
  %% \end{definition}

  %% A maximal segment is by definition a DSS.  Similarly to Feschet and
  %% Tougne {\bf [Feschet,Tougne]}, we call {\em tangential cover} the set
  %% of all maximal segments of the curve.

  %% A maximal segment is by definition a DSS. Using
  %% Proposition~\ref{prp:elementary-properties-of-DSS}.iii, one can show
  %% that maximal segments cover the digital curve.  Similarly to Feschet
  %% and Tougne {\bf [Feschet,Tougne]}, we call {\em tangential cover} the
  %% set of all maximal segments of the curve.

  %% \begin{proposition}
  %% \label{prp:maximal-segments-cover}
  %% For any maximal segment $\PTS{i}{j}$, there exist an index $k$
  %% (between $i$ and $j$) so that either $i=B(k)$ and $j=F(B(k))$ or
  %% $j=F(k)$ and $i=B(F(k))$. As a corollary, any DSS $\PTS{i}{j}$ (hence any
  %% point) belongs to at least one maximal segment, $\PTS{B(j)}{F(B(j))}$
  %% and $\PTS{B(F(i))}{F(i)}$, that may be identical.
  %% \end{proposition}

  %% We index all the maximal segments of the curve by increasing indices:
  %%   $M^{i}=\PTS{m_i}{n_i}$ with $F(m_i)=n_i$ and $B(n_i)=m_i$.  
  
  From characterizations (3) and (4) of
  Definition~\ref{def:maximal-segment}, any DSS $\PTS{i}{j}$ and hence
  any point belongs to at least two maximal segments (possibly
  identical) $\PTS{B(j)}{F(B(j))}$ and $\PTS{B(F(i))}{F(i)}$.

  \subsection{Patterns and DSS}
  We here recall a few properties about \emph{patterns} composing DSS
  and their close relations with continued fractions. They constitute
  a powerful tool to describe discrete lines with rational slopes
  \cite{Berstel97,HarWri60}. Since we are in the first octant, the
  slopes are between 0 and 1.

%%   W.l.o.g. all definitions and
%%   propositions stated below hold for DSS with slopes in the first
%%   octant (e.g. $\frac{a}{b}$ with $ 0 \leq a < b $).
  
  \begin{definition}
    Given a standard line $(a,b,\mu)$, we call
    \emph{pattern} of characteristics $(a,b)$ the succession of Freeman moves
    between any two consecutive upper leaning points. The Freeman moves
    defined between any two consecutive lower leaning points is the previous
    word read from back to front and is called the \emph{reversed pattern}.
    %%     Given two consecutive upper leaning points (say $U$ and $U'$) and
    %%     two consecutive lower leaning points (say $L$ and $L'$), the
    %%     \emph{pattern} is defined by the Freeman moves from $U$ to $L$ plus 
    %%     the Freeman moves from $L$ to $U'$. 
  \end{definition}
  A pattern $(a,b)$ embedded anywhere in the digital plane is
  obviously a DSS $(a,b,\mu)$ for some $\mu$. Since a DSS contains at
  least either two upper or two lower leaning points, a DSS
  $(a,b,\mu)$ contains at least one \emph{pattern} or one
  \emph{reversed pattern} of characteristics $(a,b)$.

%%   From previous definition, it is straightforward that a pattern of
%%   characteristics $(a,b)$ embedded anywhere in the digital plane is a
%%   DSS$(a,b,\mu)$ for some $\mu$, and that a DSS$(a,b,\mu)$ contains at least
%%   one \emph{pattern} or one \emph{reversed pattern} of characteristics $(a,b)$.

%%   \begin{proposition}
%%     \label{prop:pattern-segment}
%%     \begin{itemize} 
%%     \item[(i)] A pattern of characteristics $(a,b)$ embeded anywhere
%%       on the digital plane is a DSS$(a,b,\mu)$ for some $\mu$,
%%     \item[(ii)] a DSS$(a,b,\mu)$ contains at least one \emph{pattern} of characteristics
%%       $(a,b)$ or its reversal.      
%%     \end{itemize}
%%   \end{proposition}

  \begin{definition} 
    We call \emph{simple continued fraction} and we write:
    \[ z = a/b = [0,u_{1} \ldots, u_i, \ldots, u_n]  \quad \rm{with} \quad 
    z = 0 + \cfrac{1}{ u_1 + \cfrac{1}{\ldots +
          \cfrac{1}{ u_{n-1} + \cfrac{1}{u_n}}}} 
    \] 
    We call \emph{$k$-th convergent} the simple continued fraction
    formed of the $k+1$ first partial quotients: $z_k =
    \frac{p_{k}}{q_{k}} =[0,u_{1},\ldots, u_k]$.
  \end{definition}  
  There exists a recursive transformation for computing the pattern of
  a standard line from the \emph{simple continued fraction} of its
  slope \cite{Berstel97}. We call $E$ the mapping from the set of
  positive rationnal number smaller than one onto Freeman-code's words
  defined as follows. First terms are stated as $E(z_0) = 0$ and
  $E(z_1) = 0^{u_1}1$ and others are expressed recursively:
  \begin{eqnarray}
    E(z_{2i+1}) & = & E(z_{2i})^{u_{2i+1}} E(z_{2i-1})
    \label{pattern:rec:odd}\\ 
    E(z_{2i}) & = & E(z_{2i - 2}) E(z_{2i-1})^{u_{2i}}
    \label{pattern:rec:even}
  \end{eqnarray}

  In the following, the \emph{complexity} of a pattern is the depth of its
  decomposition in simple continued fraction. We recall a few more relations:
   \begin{eqnarray}
     p_{k}q_{k-1} - p_{k-1}q_{k} = (-1)^{k+1} \label{pattern:rec:dif} \\
     (p_{k},q_{k})=u_{k}(p_{k-1},q_{k-1})+(p_{k-2},q_{k-2}) \label{pattern:rec}
   \end{eqnarray} 
%%    \begin{eqnarray}
%%      p_{k}q_{k-1} - p_{k-1}q_{k} = (-1)^{k+1} \label{pattern:rec:dif} \\
%%      p_{k}=u_{k}p_{k-1}+p_{k-2} \label{pattern:rec:num} \\
%%      q_{k}=u_{k}q_{k-1}+q_{k-2} \label{pattern:rec:den}
%%    \end{eqnarray} 

  We now focus on computing vector relations between leaning points
  (upper and lower) inside a pattern. In the following we consider a
  DSS $(a,b,0)$ in the first octant starting at the origin and ending
  at its second lower leaning point (whose coordinate along the
  $x$-axis is positive). We define $a/b=z_{n}=[0,u_{1},\ldots, u_{n}]$
  for some $n$.  Points will be called $U_{1}$,$L_{1}$, $U_{2}$ and
  $L_{2}$ as shown in \RefFig{fig:aop-pattern}.  We can state
  $\mathbf{{U_{1}L_{1}}} = \mathbf{U_{2}L_{2}}$ and
  $\mathbf{U_{1}U_{2}} = \mathbf{L_{1}L_{2}} = \left(b,a\right)$. We
  recall that the Freeman moves of $[U_{1}L_{1}]$ are the same as
  those of $[U_{2}L_{2}]$. Furthermore Freeman moves between $U_{1}$
  and $U_{2}$ form the \emph{pattern} $(a,b)$ and those between
  $L_{1}$ and $L_{2}$ form the \emph{reversed pattern} $(a,b)$.
  
  \begin{figure}[htbp]
    \begin{center}
      \input{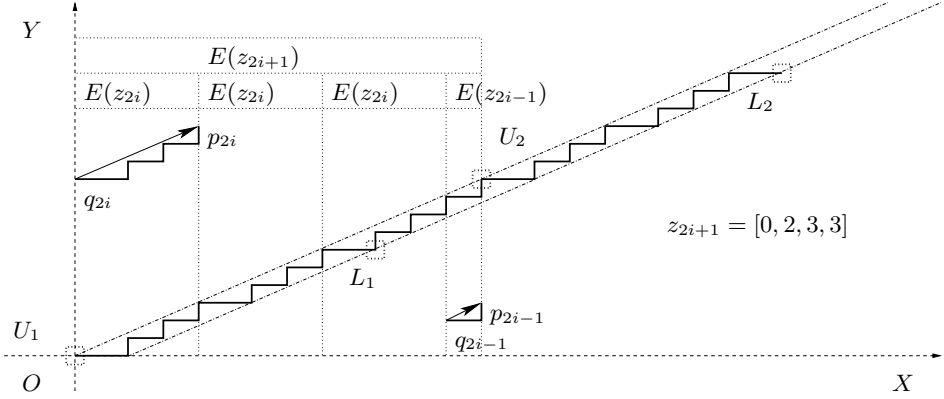}
      \caption{A DSS$(a,b,0)$ with an odd complexity slope, taken
        between origin and its second lower leaning point.}
      \label{fig:aop-pattern}
    \end{center}
  \end{figure}
  
  \begin{proposition}
    \label{prop:odd-pattern}
    A pattern with an odd complexity (say $n=2i+1$) is such that
    $\mathbf{U_{1}L_{1}}= (u_{2i+1}-1) \VEC{q_{2i}}{p_{2i}} +
    \VEC{q_{2i-1}}{p_{2i-1}} +\VEC{1}{-1}$ and $\mathbf{L_{1}U_{2}} =
    \VEC{q_{2i}-1}{p_{2i} + 1}$.  Moreover the DSS $[U_{1}L_{1}]$ has
    $E(z_{2i})^{u_{2i+1} -1}$ as a left factor, and the DSS
    $[L_{1}U_{2}]$ has $E(z_{2i-1})^{u_{2i}}$ as a right factor.
  \end{proposition}
  \begin{proof}
    From \Eq{pattern:rec:dif} we have: $p_{2i+1}q_{2i} -
    p_{2i}q_{2i+1} = (-1)^{2i+1+1} = 1$, which can be rewritten as: $a
    q_{2i} - b p_{2i} = 1$. $(q_{2i},p_{2i})$ are clearly the B\'ezout
    coefficients of $(a,b)$. 
    % JOL
    One can check that point $(b + 1
    -q_{2i},a -1 - p_{2i})$ is $L_1$: its remainder is $a+b-1$ and its
    $x$-coordinate while positive is smaller than $b$. We immediately
    get $\mathbf{U_{1}L_{1}} = \VEC{b + 1 -q_{2i}}{a -1 - p_{2i}} $.
    %
%%     The remainder of $(b + 1 -q_{2i},a -1 -
%%     p_{2i})$ is $a+b-1$. Since $b + 1 -q_{2i}$ is positive but smaller
%%     than $b$, it is the first positive lower leaning point $L_1$ and
%%     $\mathbf{U_{1}L_{1}} = \VEC{b + 1 -q_{2i}}{a -1 - p_{2i}} $.
    Using \Eq{pattern:rec} yields: $\mathbf{U_{1}L_{1}}=\VEC{
    (u_{2i+1}-1) q_{2i} + q_{2i-1} +1}{(u_{2i+1} -1) p_{2i} + p_{2i-1}
    -1 }$. From $\mathbf{L_{1}U_{2}} = -\mathbf{U_{1}L_{1}} +
    \mathbf{U_{1}U_{2}}$, we further get that $\mathbf{L_{1}U_{2}} =
    \VEC{q_{2i} -1}{p_{2i}+1}$.  From \Eq{pattern:rec:odd}
    $E(z_{2i})^{u_{2i+1} -1}$ is a left factor of $[U_{1}U_{2}]$ but
    also of $[U_{1}L_{1}]$. Writing $E(z_{2i+1})$ as
    $E(z_{2i})^{u_{2i+1}-1}E(z_{2i-2})E(z_{2i-1})^{u_{2i}+1}$, and
    expanding $\mathbf{L_{1}U_{2}}$ as $\VEC{ u_{2i}q_{2i-1} +
    q_{2i-2} - 1 }{u_{2i}p_{2i-1} + p_{2i-2} +1}$ with
    \Eq{pattern:rec}, we see that $E(z_{2i-1})^{u_{2i}}$ is a right
    factor of $[L_{1}U_{2}]$.\qed
  \end{proof}
  \begin{proposition}
    \label{prop:even-pattern}
    A pattern with an even complexity (say $n=2i$) is such that
    $\mathbf{U_{1}L_{1}}=\VEC{q_{2i-1} + 1}{p_{2i-1} -1 }$ and
    $\mathbf{L_{1}U_{2}}= (u_{2i}-1)\VEC{q_{2i-1}}{p_{2i-1}} + \VEC{
    q_{2i-2}}{p_{2i-2}} + \VEC{-1}{1}$.  Moreover the DSS
    $[U_{1}L_{1}]$ has $E(z_{2i-2})^{u_{2i-1}}$ as a left factor, and
    the DSS $[L_{1}U_{2}]$ has $E(z_{2i-1})^{u_{2i}-1}$ as a right
    factor.
  \end{proposition}

    \section{Properties of maximal segments for convex curves}

    In this section, we study relations between maximal segments and
    digital edges of convex shape digitization.  The dilation of $S$
    by a real factor $r$ is denoted by $r \cdot S$. Let
    $\mathcal{D}_{m}$ be the digitization of step $1/m$, i.e. if $S$
    is a real shape: $\mathcal{D}_{m}(S)=(m \cdot S) \cap
    \mathbb{Z}^{2}$.  The length estimator based on the
    city-block distance is written as $\mathcal{L}^1$. 
    
    \subsection{Convex digital polygon (CDP)}
    \begin{definition}
      A \emph{convex digital polygon (CDP)} $\Gamma$ is a subset of
      the digital plane equal to the digitization of its convex hull,
      i.e. $\Gamma = \mathcal{D}_1(\Conv{\Gamma})$. Its {\em vertices}
      $(V_i)_{i=1..e}$ form the minimal subset for which
      $\Gamma=\GDigit{\Conv{V_1, \ldots,V_e}}$. The points on the
      boundary of $\Gamma$ form a 4-connected contour. The number of
      vertices (or edges) of $\Gamma$ is denoted by $n_{e}(\Gamma)$
      and its perimeter by $\mathrm{Per}(\Gamma)$.
%%       $\Gamma$ is a \emph{convex digital polygon (CDP)} if its
%%       vertices $(V_i)_{i=1..e}$ form the minimal set of discrete
%%       points such that $\Gamma=\GDigit{\Conv{V_1, \ldots,V_e}}$ and
%%       $\Gamma$ is different from the digitization of the convex hull
%%       of any proper subset of the $(V_i)$. The number of vertices or
%%       edges of $\Gamma$ is denoted by $n_{e}(\Gamma)$ and its
%%       perimeter by $\mathrm{Per}(\Gamma)$.
    \end{definition}
    A $CDP$ is also called a lattice convex polygon \cite{Voss93}. An
    \emph{edge} is the Euclidean segment joining two consecutive
    vertices, and a \emph{digital edge} is the discrete segment
    joining two consecutive vertices. It is clear that we have as many
    \emph{edges} as \emph{digital edges} and as vertices.  From
    characterizations of discrete convexity \cite{KIM:PAMI:1982}, we
    clearly see that:
    
    \begin{proposition}
      \label{prop:edge-pattern}
      Each digital edge of a CDP is either a pattern or a succession of the
      same pattern whose slope is the one of the edge. In other words, both
      vertices are upper leaning points of the digital edge.
    \end{proposition}
%%     \begin{proof}
%%       Let $D$ be an $CDP$ and say $V_{i}$ and $V_{i+1}$ two consecutive
%%       vertices. From \cite{KIM:PAMI:1982}, we know that $V_{i}V_{i+1}$ is a
%%       DSS$(a,b,\mu)$. It contains at least a \emph{pattern} or a
%%       \emph{reversed pattern}. Considering the first case there is a least two
%%       upper leaning points (on the line of equation $aX - bY = \mu$ ) on the
%%       digital edge. Assume the vertices of the digital edge are not both upper
%%       leaning points, then the edge linking the vertices is below one the
%%       previous upper leaning points. The area property is thus not fulfilled,
%%       so the region is not digitally convex.  On the other hand, the second
%%       case leads to only one upper leaning point on the digital edge.  Assume
%%       none of the vertices are upper leaning points , then there is one point
%%       strictly above the edge (the upper leaning point), which leads to a
%%       contradiction. Assume one of the vertices is the upper leaning point,
%%       then one of the lower leaning point is not on the digital edge. This is
%%       clearly a contradiction.\qed
%%     \end{proof}
  
    We now recall one theorem concerning the asymptotic number of
    vertices of CDP that are digitization of continuous shapes. It
    comes from asymptotic properties of random polytopes.
    \begin{theorem}{(Adapted from Balog, B\'{a}r\'{a}ny \cite{Balog1991})}
      \label{thm:Balog:nbV1}
      If $S$ is a plane convex body with $\mathcal{C}^{3}$ boundary
      and positive curvature then $\mathcal{D}_{m}(S)$ is a CDP and
      \[
      c_{1}(S) m^{\frac{2}{3}} \leq n_{e}(\mathcal{D}_{m}(S))
      \leq c_{2}(S) m^{\frac{2}{3}}
%%       c_{1}(S) m^{\frac{2}{3}} \leq n_{v}(\Conv{\mathcal{D}_{m}(S)})
%%       \leq c_{2}(S) m^{\frac{2}{3}}
      \]
      where the constants $c_{1}(S)$ and $c_{2}(S)$ depend on extremal
      bounds of the curvatures along $S$. Hence for a disc $c_{1}$ and
      $c_{2}$ are absolute constants.
%      Given a disc of radius $m$
%      (say disc(m)) in the plane we have :
%      \[
%      c_{1} m^{\frac{2}{3}} \leq n_{v}(\Conv{\GDigit{disc(m)}}) \leq c_{2}
%      m^{\frac{2}{3}}
%      \]
%      where $c_{1}$ and $c_{2}$ are absolute constants. 
    \end{theorem}
    
    \subsection{Links between maximal segments and edges of CDP }    
    Maximal segments are DSS: between any two upper (resp. lower)
    leaning points lays at least a lower (resp. upper) leaning
    point. The slope of a maximal segment is then defined by two consecutive upper and/or
    lower leaning points. Digital edges are patterns
    and their vertices are upper leaning points (from
    Prop. \ref{prop:edge-pattern}). Thus, vertices may be upper
    leaning points but never lower leaning points of maximal
    segments. We have %Since a digital edge is a DSS, we get
    \begin{lemma}
      \label{lem:msInde}
      A maximal segment cannot be strictly contained into a digital edge.
    \end{lemma}
    We now introduce a special class of digital edge.
    \begin{definition}
      \label{def:supporting-edge}
      We call \emph{supporting edge}, a digital edge whose two
      vertices define leftmost and rightmost upper leaning points of a
      maximal segment. 
%% JOL: mis dans lemme
%%      This maximal segment is the only one containing
%%      the digital edge and has the same slope.
    \end{definition}
%%     The following result is straighforward since if the supporting edge were
%%     defining two different maximal segments, they would extend each others.
    Relations between maximal DSS and digital edges are given by the
    following lemmas:
    \begin{lemma}
      \label{lem:msSE}
      A supporting edge defines only one maximal segment: it is the
      only one containing the edge and it has the same slope.  If a
      maximal segment contains two or more upper leaning points then
      there is a \emph{supporting edge} linking its leftmost and
      rightmost upper leaning points with the same slope.  If a
      maximal segment contains three or more lower leaning points then
      it has a \emph{supporting edge}.
    \end{lemma}
%%     \begin{proof}
%%       Between three lower leaning points, lay at least
%%       two upper leaning points, thus a digital edge with the same slope,
%%       which is by definition the \emph{supporting edge} defining it.\qed
%%     \end{proof}
    \begin{lemma}
      \label{lem:ms2llp}
      If a maximal segment is defined by only two consecutive lower
      leaning points then it has one upper leaning point which is some
      vertex of the CDP by convexity.
    \end{lemma}
%%     \begin{proof}
%%       From lemma \ref{lem:msInde}, this maximal segment
%%       contains at least one vertex. Between two consecutive lower
%%       leaning points lays only one upper leaning point, which is a
%%       vertex by convexity.\qed
%%     \end{proof}

    Lengths of maximal segments and digital edges are
    tightly intertwined, as shown by the two next propositions.
%%     We now focus on relations between maximal segments and edges of a
%%     digital convex polygon so as to get bounds on number or length.

    \begin{proposition} 
      \label{prop:lMS-SE}
      Let $[V_{k}V_{k+1}]$ be a \emph{supporting edge} of slope
       $\frac{a}{b}$ made of $f$ patterns $(a,b)$ and let $MS$ be the
       \emph{maximal segment} associated with it
       (\RefLemma{lem:msSE}). Their lengths are linked by the
       inequalities:
      \[
        {\mathcal{L}^{1}}(V_{k}V_{k+1}) \leq {\mathcal{L}^{1}}(MS)
        \leq \frac{f+2}{f} {\mathcal{L}^{1}}(V_{k}V_{k+1}) - 2 
        \quad\mbox{and} 
        \quad 
        \frac{1}{3} \mathcal{L}^{1}(MS) \leq
        \mathcal{L}^{1}(V_{k}V_{k+1}) \leq \mathcal{L}^{1}(MS) \leq 3
        \mathcal{L}^{1}(V_{k}V_{k+1})
        \]
    \end{proposition}
    \begin{proof}
      Vertices $V_{k}$ and $V_{k+1}$ are {\em leftmost} and {\em
      rightmost} upper leaning points of $MS$. The points
      $V_{k}-(b,a)$, $V_{k+1}+(b,a)$ while clearly upper leaning points of the
      standard line going through $[V_{k} V_{k+1}]$ cannot belong to
      the CDP. Hence $MS$ cannot extend further of its supporting
      edge of more than $|a| + |b| - 1$ points on both
      sides. Consequently $ {\mathcal{L}^{1}}(MS) \leq
      {\mathcal{L}^{1}}(V_{k}V_{k+1}) +2(|a| + |b| - 1)$. Using
      ${\mathcal{L}^{1}}(V_{k}V_{k+1}) = f(|a| + |b|)$ brings:
      ${\mathcal{L}^{1}}(V_{k}V_{k+1}) \leq {\mathcal{L}^{1}}(MS))
      \leq \frac{f+2}{f} {\mathcal{L}^{1}}(V_{k}V_{k+1}) - 2$. Worst
      cases bring $\mathcal{L}^{1}(V_{k}V_{k+1}) \leq
      \mathcal{L}^{1}(MS) \leq 3 \mathcal{L}^{1}(V_{k}V_{k+1})$ \qed
    \end{proof}
%%     \begin{corollary} \label{cor:lMS-SE}
%%       Any supporting edge $E_{k}$ has the length of its maximal segment
%%       $MS_{k'}$ bounded by three times the length of $E_{k}$.
%%       \[
%%       \frac{1}{3} \mathcal{L}^{1}(MS) \leq
%%       \mathcal{L}^{1}(V_{k}V_{k+1}) \leq \mathcal{L}^{1}(MS) \leq 3
%%       \mathcal{L}^{1}(V_{k}V_{k+1})
%%       \]
%%     \end{corollary}
    \begin{proposition}
      \label{prp:lMS-LUL}
      Let $MS_{k'}$ be a maximal segment in the configuration of
      \RefLemma{lem:ms2llp}, and so let $V_k$ be the vertex that is
      its upper leaning point. The length of the maximal segment is
      upper bounded by:
%%       Let $V_{k-1}$, $V_{k}$ and $V_{k+1}$ three consecutive vertices
%%       that meet the configuration of lemma \ref{lem:ms2llp}. Let
%%       $MS_{k'}$ be a maximal segment defined by its \emph{reversed
%%       pattern} with $V_{k}$ its upper leaning point. We have:
      \[
      \mathcal{L}^{1}(MS_{k'}) \leq 4 \left(
      {\mathcal{L}^{1}}(V_{k-1}V_{k}) +
      {\mathcal{L}^{1}}(V_{k}V_{k+1}) \right)
      \]
    \end{proposition}
    \begin{proof}
      We call $L_{1}$, $L_{2}$ the leftmost and rightmost lower
      leaning points and $U_{2}\equiv V_{k}$ the upper leaning point
      (see \RefFig{fig:aop-pattern}).  Suppose that $MS_{k'}$ has a
      slope with an odd complexity (say $2i+1$).
      \RefProposition{prop:odd-pattern} implies
      ${\mathcal{L}^{1}}(\mathbf{L_{1}U_{2}}) = q_{2i} +
      p_{2i}$. There is clearly a right part of $[L_{1}U_{2}]$
      (i.e. $[L_{1}V_{k}]$) that is contained in $[V_{k-1}V_{k}]$ and
      touches $V_k$. The pattern $E(z_{2i-1})^{u_{2i}}$ is a right
      factor of $[L_{1}U_{2}]$ (\RefProposition{prop:odd-pattern}
      again). It is indeed a right factor of $[V_{k-1}V_{k}]$ too,
      since it cannot extends further than $V_{k-1}$ to the left
      without defining a longer digital edge. We get $[V_{k-1}V_{k}]
      \supseteq E(z_{2i-1})^{u_{2i}}$ and immediately 
      ${\mathcal{L}^{1}}(V_{k-1}V_{k}) \geq
      u_{2i}\mathcal{L}^{1}(E(z_{2i-1})) =
      u_{2i}(q_{2i-1}+p_{2i-1})$.
%%       the length inequality
%%       ${\mathcal{L}^{1}}(V_{k-1}V_{k}) \geq
%%       u_{2i}(q_{2i-1}+p_{2i-1})$.

%%       From Prop. \ref{prop:odd-pattern} we get
%%       ${\mathcal{L}^{1}}(\mathbf{L_{1}U_{2}}) = q_{2i} + p_{2i}$ and
%%       ${\mathcal{L}^{1}}(\mathbf{U_{2}L_{2}}) = (u_{2i+1}
%%       -1)(q_{2i}+p_{2i}) + q_{2i-1} + p_{2i-1}$.  Moreover the DSS
%%       between $L_{1}$ and $U_{2}$ has $E(z_{2i-1})^{u_{2i}}$ as a
%%       right factor.  The digital edge $V_{k-1}V_{k}$ must contains
%%       $E(z_{2i-1})^{u_{2i}}$ as a Freeman move made of the same
%%       pattern, and $V_{k-1}V_{k} \supseteq E(z_{2i-1})^{u_{2i}}$
%%       giving: ${\mathcal{L}^{1}}(V_{k-1}V_{k}) \geq
%%       u_{2i}(q_{2i-1}+p_{2i-1})$.  Same reasoning applied to
%%       $V_{k}V_{k+1}$ brings: ${\mathcal{L}^{1}}(V_{k}V_{k+1}) \geq
%%       (u_{2i+1} -1 )( q_{2i} + p_{2i} )$.
    
      From \Eq{pattern:rec}, we have: $q_{2i} + p_{2i} =
      u_{2i}(q_{2i-1}+p_{2i-1}) + q_{2i-2}+p_{2i-2}$ and $q_{2i-2} +
      p_{2i-2} \leq q_{2i-1} + p_{2i-1}$.  We obtain immediately
      ${\mathcal{L}^{1}}(\mathbf{L_{1}U_{2}}) = q_{2i} + p_{2i} \leq
      (u_{2i}+1)(q_{2i-1}+p_{2i-1})$. By comparing this length to the
      length of the digital edge $[V_{k-1}V_{k}]$, we get
      ${\mathcal{L}^{1}}(\mathbf{L_{1}U_{2}}) \leq
      \frac{u_{2i}+1}{u_{2i}} {\mathcal{L}^{1}}(V_{k-1}V_{k})$.

      \RefProposition{prop:odd-pattern} and similar arguments on
      $[V_{k}V_{k+1}]$ brings ${\mathcal{L}^{1}}(\mathbf{U_{2} L_{2}})
      \leq \frac{u_{2i+1}}{u_{2i+1}-1}
      {\mathcal{L}^{1}}(V_{k-1}V_{k})$.
%%       From \Eq{pattern:rec}, we have:
%%       $q_{2i} + p_{2i} = u_{2i}(q_{2i-1}+p_{2i-1}) +
%%       q_{2i-2}+p_{2i-2}$ and $q_{2i-2} + p_{2i-2} \leq q_{2i-1} +
%%       p_{2i-1}$.  There is clearly a part of $[L_{1}U_{2}]$
%%       that is included into the edge $V_{k-1}V_{k}$.  This part is at
%%       least the right factor $E(z_{2i})$ of $[L_{1}U_{2}]$. This right
%%       factor is indeed also included into the digital edge
%%       $V_{k-1}V_{k}$, otherwise this factor would define a longer
%%       edge. Thus ${\mathcal{L}^{1}}(\mathbf{L_{1}U_{2}}) \leq
%%       (u_{2i}+1)(q_{2i-1}+p_{2i-1})$. For
%%       ${\mathcal{L}^{1}}(\mathbf{U_{2}L_{2}})$, we get a direct upper
%%       bound from its definition, leading to:
%%       ${\mathcal{L}^{1}}(\mathbf{L_{1}U_{2}}) \leq
%%       \frac{u_{2i}+1}{u_{2i}} {\mathcal{L}^{1}}(V_{k-1}V_{k})$ and
%%       ${\mathcal{L}^{1}}(\mathbf{UL_{2}}) \leq
%%       \frac{u_{2i+1}}{u_{2i+1}-1} {\mathcal{L}^{1}}(V_{k-1}V_{k})$.
      Worst cases are then ${\mathcal{L}^{1}}(\mathbf{L_{1}U_{2}})
      \leq 2 {\mathcal{L}^{1}}(V_{k-1}V_{k})$ and
      ${\mathcal{L}^{1}}(\mathbf{U_{2}L_{2}}) \leq 2
      {\mathcal{L}^{1}}(V_{k}V_{k+1})$. The case where $MS_{k'}$ has
      a slope with an even complexity (say $2i$) uses
      Prop. \ref{prop:even-pattern} and is treated similarly.
%      for the same reasons as above
%      leads to ${\mathcal{L}^{1}}(\mathbf{L_{1}U_{2}}) \leq
%      \frac{u_{2i}}{u_{2i}-1} {\mathcal{L}^{1}}(V_{k-1}V_{k}) \leq 2
%      {\mathcal{L}^{1}}(V_{k-1}V_{k})$ and ${\mathcal{L}^{1}}(\mathbf{U_{2}L_{2}})
%      \leq \frac{u_{2i-1}+1}{u_{2i-1}} {\mathcal{L}^{1}}(V_{k}V_{k+1}) \leq 2
%      {\mathcal{L}^{1}}(V_{k}V_{k+1})$.
    
      Since $MS$ has only one upper leaning point, it
      cannot be extended further than
      ${\mathcal{L}^{1}}(\mathbf{U_{2}L_{2}})$ on the left and
      ${\mathcal{L}^{1}}(\mathbf{L_{1}U_{2}})$ on the right
      (\RefLemma{lem:msSE} ). We thus get
      ${\mathcal{L}^{1}}(MS_{k'}) \leq
      4({\mathcal{L}^{1}}(V_{k-1}V_{k})+{\mathcal{L}^{1}}(V_{k}V_{k+1})).$\qed
%%       From lemma \ref{lem:msSE} maximal segment in a $L_{1}U_{2}L_{2}$
%%       position cannot be extended more than
%%       ${\mathcal{L}^{1}}(\mathbf{U_{2}L_{2}})$ on the left and
%%       ${\mathcal{L}^{1}}(\mathbf{L_{1}U_{2}})$ on the right. As a
%%       result we get ${\mathcal{L}^{1}}(MS_{k'}) \leq
%%       4({\mathcal{L}^{1}}(V_{k-1}V_{k})+{\mathcal{L}^{1}}(V_{k}V_{k+1}))$\qed
    \end{proof}
    
    A proof of the following theorem based on pattern analysis is
    given in Appendix B for limited space reasons. A similar result
    related to linear integer programming is in
    \cite{Shevchenko81}. It may also be obtained by viewing standard
    lines as intersection of two knapsack polytopes \cite{Hayes83}.
    
    \begin{theorem} \label{thm:msMaxedge}
      Let $E$ be a supporting edge whose slope has a complexity $n$,
      $n \ge 2$, then the maximal segment containing $E$ includes at
      most $n$ other edges on each side of $E$.
%%       Let $E$ be a supporting edge of a CDP whose slope $\frac{a}{b}$
%%       equals $[u_0, \ldots, u_i, \ldots, u_n]$ with $n \geq 2$ and
%%       containing one or more pattern(s) of characteristics $(a,b)$, any
%%       maximal segment containing $E$ may contains at most $n$ edges on
%%       each side of $E$.
    \end{theorem}
    
    \begin{corollary} \label{cor:msMaxedge}
      The shortest pattern of a supporting edge for which its maximal
      segment may contain $2n+1$ digital edge is $z_{n} =
      [0,2,\ldots,2]$. If the DCP is enclosed in a $m \times m$ grid,
      then the maximal number $n$ of digital edges included in one
      maximal segment is upper bounded as: 
      $ n \leq {\log{
      \frac{4 m}{\sqrt{2}}}} / {\log{ (1+\sqrt{2})}} - 1$.
%%       $ n \leq \frac { \log{
%%       \frac{4 m}{\sqrt{2}}}}{\log{ (1+\sqrt{2})}} - 1$.
    \end{corollary} 
    \begin{proof}
    The number $L=[0,2,\ldots,2,\ldots]$ is a quadratic number equal
    to $-1+\sqrt{2}$. Its recursive characterization is $ U_{n} = 2
    U_{n-1} + U_{n-2}$ with $U_{0} = 0$ and $U_{1} = 1$.  Solving it
    leads to $U_{n} = \frac{\sqrt{2}}{4} \left( (1+\sqrt{2})^{n} -
    (1-\sqrt{2})^{n} \right)$. Hence asymptotically, $ U_{n} \approx
    \frac{\sqrt{2}}{4} (1+\sqrt{2})^{n}$ and
    $\lim_{n \rightarrow \infty }{\frac{U_{n}}{U_{n+1}}} = L$.
    
    The shortest edge of slope complexity $n$ is clearly an $n$-th
    convergent of $L$. To fit into an $m \times m$ grid, the
    complexity $n$ is such that $U_{n+1} \leq m$. We thus obtain that
    $ n \leq { \log{ \frac{4 m}{\sqrt{2}}}}/ {\log{ (1+\sqrt{2})}} -
    1$.\qed
    \end{proof}
    
    %%%%%%%%%%%%%%%%%%%%%%%%%%%%%%%%%%%%%%%%%%%%%%%%%%%%%%%%%%%%%%%%%%%%%%%%%%%%%%%%%
    %%%%%%%%%%%%%% EXTENDED VERSION %%%%%%%%%%%%%%%%%%%%%%%%%%%%%%%%%%%%%%%%%%%%%%%%%
    %%%%%%%%%%%%%%%%%%%%%%%%%%%%%%%%%%%%%%%%%%%%%%%%%%%%%%%%%%%%%%%%%%%%%%%%%%%%%%%%%
    %  FDV begin %%%%%%%%%%%%%%%%%%%%%%%%%%%%%%%%%%%%%%%%%%%%%%%%%%%%%%%%%%%%%%%%%%%%
    %%%%%%%%%%%%%%%%%%%%%%%%%%%%%%%%%%%%%%%%%%%%%%%%%%%%%%%%%%%%%%%%%%%%%%%%%%%%%%%%%
    
%%     \begin{proposition} \label{prop:MsLulSameQuotients}
%%       If two maximal segments in the configuration of
%%       \RefLemma{lem:ms2llp} have the same upper leaning point (say
%%       $V_{k}$) and a complexity of same parity (say $2i$ and $2j$)
%%       their slopes have the same $min(2i,2j)+1$ first partial
%%       quotients.
%%     \end{proposition}
%%     \begin{proof}
%%       We first suppose that slopes have an even parity. Let us call
%%       $MS_{1}$ the maximal segment which slope equals $z_{2i}$ and
%%       $MS_{2}$ the one whose slope equals $z'_{2j}$.  W.l.o.g. we
%%       consider that $2j = 2i + 2p$ with $p>0$. Since their slopes
%%       have same parity the position of lower leaning points 
      
%%       One of the maximal segment is necessarily the bigger of the two.
      
    \begin{proposition} 
      There exists at most two maximal segments per vertices in the
      configuration of \RefLemma{lem:ms2llp} with different parities
      of complexity.
    \end{proposition}
    \begin{proof}
      We first prove that there is at most one maximal segment with
      only one upper leaning point on a vertice of a DCP with an even
      complexity.  
      
      Let us suppose that $MS_{1}$ and $MS_{2}$ are two maximal
      segments sharing a vertice of the CDP (say $U_{2}$) with their
      slopes

    \end{proof}

    \subsection{Asymptotic number and size of maximal segments}
    %%%%%%%%%%%%%%%%%%%%%%%%%%%%%%%%%%%%%%%%%%%%%%%%%%%%%%%%%%%%%%%%%%%%%%%%%%%%%%%
    % JOL BEGIN
    %%%%%%%%%%%%%%%%%%%%%%%%%%%%%%%%%%%%%%%%%%%%%%%%%%%%%%%%%%%%%%%%%%%%%%%%%%%%%%%
    We assume in this section that the digital convex polygon $\Gamma$
    is enclosed in a $m \times m$ grid. We wish to compute a lower
    bound for the number of edges related to at least one maximal
    segment. We show in \RefTheorem{thm:number-of-edges-wrt-MS} that
    this number is significant and increases at least as fast as the
    number of edges of the DCP divided by $\log m$. From this lower
    bound, we are able to find an upper bound for the length of the
    smallest maximal segment of a DCP
    (\RefTheorem{thm:length-smallest-ms}). We first label each vertex
    of the DCP as follows: (i) a {\em 2-vertex} is an upper leaning
    point of a supporting edge, (ii) a {\em 1-vertex} is an upper
    leaning point of some maximal segment but is not a 2-vertex, (iii)
    {\em 0-vertices} are all the remaining vertices.  The number of
    $i$-vertices is denoted by $n_i$. Given an orientation on the
    digital contour, the number of edges going from an $i$-vertex to a
    $j$-vertex is denoted by $n_{ij}$.
    
    \begin{theorem}
      \label{thm:number-of-edges-wrt-MS}
      The number of supporting edges and of 1-vertices of $\Gamma$ are
      related to its number of edges with
      \begin{equation}
        \frac{n_e(\Gamma)}{\Omega(\log m)} \le n_1 + 2n_{22}.
      \end{equation}
      An immediate corollary is that there are at least
        $n_e(\Gamma)/{\Omega(\log m)}$ maximal segments.
    \end{theorem}
    \begin{proof}
%% From \RefTheorem{thm:msMaxedge} and its corollary ~\ref{cor:msMaxedge}, we
%%       know that a DSS hence a maximal segment cannot include more than
%%       $\Omega(\log m)$ edges. Hence there cannot be more than $\Omega(\log m)$
%%       0-vertices for one 1-vertex or for one 2-vertex. We get $n_{00} \le (n_1 +
%%       n_2) \Omega(\log m)$.  We develop the number of edges with each possible
%%       label, noting that a 2-vertex cannot be isolated by definition:
%%       \begin{eqnarray*}
%%        n_e & = & n_{22} + \underbrace{n_{02}+n_{12}}_{\le n_{22}} +
%%         \underbrace{n_{20}+n_{21}}_{\le n_{22}} +n_{00} +
%%         \underbrace{n_{01}+n_{10}+n_{11}}_{\le 3n_{1}} \\
%%        &\le& 3n_{22} + \underbrace{n_{00}}_{\le (n_1 + n_2) \Omega(\log m)} +3 n_1 \\
%%        &\le& 3n_{22} + (n_1 + \underbrace{n_2}_{\le 2n_{22}})\Omega(\log m)  +3 n_1 \\
%%        &\le& \Omega(\log m)(n_1+2n_{22})
%%       \end{eqnarray*}
%%       This concludes the proof. \qed
      From \RefTheorem{thm:msMaxedge} and its
      \RefCorollary{cor:msMaxedge}, we know that a DSS hence a maximal
      segment cannot include more than $\Omega(\log m)$ edges. Hence
      there cannot be more than $\Omega(\log m)$ 0-vertices for one
      1-vertex or for one 2-vertex. We get $n_{00} \le (n_1 + n_2)
      \Omega(\log m)$.  We develop the number of edges with each
      possible label: $n_e(\Gamma) = n_{22} + n_{02}+n_{12} + n_{20}+
      n_{21} + n_{00} + n_{01}+n_{10}+n_{11}$. Since, $n_{02}+n_{12}
      \le n_{22}$, $n_{20}+n_{21} \le n_{22}$ and
      $n_{01}+n_{10}+n_{11} \le 3n_{1}$, we get $n_e(\Gamma) \le
      3n_{22} + n_{00} +3 n_1$. Noting that a 2-vertex cannot be
      isolated by definition of supporting edges
      (\RefDefinition{def:supporting-edge}) gives $n_2 \le
      2n_{22}$. Once inserted in $n_{00} \le (n_1 + n_2) \Omega(\log
      m)$ and compared with $n_e(\Gamma)$, we get the expected result.
%%       \[
%%       n_e(\Gamma) \le \Omega(\log m)(n_1+2n_{22}). 
%%       \] \qed
      \qed
    \end{proof}
    
    We now relate the DCP perimeter to the length of maximal segments.
    
    \begin{theorem}
      \label{thm:length-smallest-ms}
      The length of the smallest maximal segment of the DCP $\Gamma$
      is upper bounded:
      \begin{equation}
        \min_{l} \Lnorm{MS_{l}} \le \Omega(\log
        m)\frac{\mathrm{Per}(\Gamma)}{n_e(\Gamma)}.
      \end{equation}
    \end{theorem}
    \begin{proof}
      We have $Per(\Gamma) = \sum_{n_{e}}\Lnorm{E_{i}}$. We now may
      expand the sum on supporting edges (22-edges), on edges touching a
      1-vertex, and on others.  Edges touching 1-vertices may be counted
      twice, therefore we divide by 2 their contribution to the total
      length.
      \begin{eqnarray}
        \label{equ:per-1}
        \sum_{n_{e}}\Lnorm{E_{i}} & \ge & \sum_{n_{22}} \Lnorm{E^{22}_{j}} +
        \frac{1}{2}\sum_{n_1} \Lnorm{E^{?1}_{k-1}} + \Lnorm{E^{1?}_{k}}
      \end{eqnarray}
      For the first term, each supporting edge indexed by $j$ (a $22$-edge)
      has an associated maximal segment, say indexed by $j'$. From
      \RefProposition{prop:lMS-SE}, we know that $\Lnorm{E^{22}_{j}} \ge
      \frac{1}{3}\Lnorm{MS_{j'}}$.
      
      For the second term, each 1-vertex indexed by $k$ is an upper leaning
      point of some maximal segment indexed by
      $k'$. \RefProposition{prp:lMS-LUL} holds and $\Lnorm{E^{?1}_{k-1}} +
      \Lnorm{E^{1?}_{k}} \ge \frac{1}{4}\Lnorm{MS_{k'}}$.
      
      Putting everything together in \Equ{equ:per-1}, we get:
      %  \begin{eqnarray*}
      \[
      %\label{equ:per-2}
      \sum_{n_{e}}\Lnorm{E_{i}}  \ge 
      \frac{1}{3}\sum_{n_{22}}\Lnorm{MS_{j'}}
      + \frac{1}{8}\sum_{n_{1}}\Lnorm{MS_{k'}}
      \ge  (\frac{1}{3} n_{22} + \frac{1}{8} n_{1}) \min_{l} \Lnorm{MS_{l}}
      \ge  \frac{1}{8}( n_{1} + 2n_{22}) \min_{l} \Lnorm{MS_{l}}
      \]
      %  \end{eqnarray*}
      
      Inserting the lower bound of \RefTheorem{thm:number-of-edges-wrt-MS}
      into the last inequality concludes. \qed
    \end{proof}
    
%%%%%%%%%%%%%%%%%%%%%%%%%%%%%%%%%%%%%%%%%%%%%%%%%%%%%%%%%%%%%%%%%%%%%%%%%%%%%%%
% JOL END
%%%%%%%%%%%%%%%%%%%%%%%%%%%%%%%%%%%%%%%%%%%%%%%%%%%%%%%%%%%%%%%%%%%%%%%%%%%%%%%

\section{Asymptotic properties of shapes digitized at increasing resolutions}

We may now turn to the main interest of the paper: studying the
asymptotic properties of discrete geometric estimators on digitized
shapes. We therefore consider a plane convex body $S$ which is
contained the square $[0,1] \times [0,1]$ (w.l.o.g.). Furthermore, we
assume that its boundary $\gamma=\partial S$ is $\mathcal{C}^3$ with
everywhere strictly positive curvature.  This assumption is not very
restrictive since people are mostly interested in regular
shapes. Furthermore, the results of this section remains valid if the
shape can be divided into a {\em finite} number of convex and concave
parts; each one is then treated separately. The digitization of $S$
with step $1/m$ defines a digital convex polygon $\Gamma(m)$ inscribed
in a $m \times m$ grid. We first examine the asymptotic behavior of
the maximal segments of $\Gamma(m)$, both theoretically and
experimentally. We then study the {\em asymptotic convergence} of a
discrete curvature estimator. 
%% We show that contrary to what was
%% thought its convergence is still an open problem.  Experimental
%% evaluation confirms this result.

    \subsection{Asymptotic behavior of maximal segments}
    The next theorem summarizes the asymptotic size of the
    smallest maximal segment wrt the grid size $m$.
    
    \begin{theorem}
      \label{thm:asymptotic-size-ms}
      The length of the smallest maximal segment of $\Gamma(m)$ has the
      following asymptotic upper bound:
      \begin{equation}
        \min_{i} \Lnorm{MS_{i}(\Gamma(m))} \le \Omega(m^{1/3} \log
        m)
      \end{equation}
    \end{theorem}
    \begin{proof}
      \RefTheorem{thm:length-smallest-ms} gives for the DCP $\Gamma(m)$
      the inequality $\min_{i} \Lnorm{MS_{i}(\Gamma(m))} \le \Omega(\log
      m)\frac{\mathrm{Per}(\Gamma(m))}{n_e(\Gamma(m))}$. Since $\Gamma(m)$
      is convex included in the subset $m \times m$ of the digital plane,
      its perimeter $\mathrm{Per}(\Gamma(m))$ is upper bounded by $4m$. On
      the other hand, \RefTheorem{thm:Balog:nbV1} indicates that its
      number of edges $n_e(\Gamma(m))$ is lower bounded by
      $c_1(S)m^{2/3}$. Putting everything together gives $\min_{i}
      \Lnorm{MS_{i}(\Gamma(m))} \le \Omega(\log m)
      \frac{4m}{c_1(S)m^{2/3}}$ which is once reduced what we wanted to
      show.
      \qed
    \end{proof}
    
    Although there are points on a shape boundary around which maximal
    segments grow as fast as $O(m^{1/2})$ (the critical points in
    \cite{Lachaud05c}), some of them do not grow as fast. A closer
    look at the proofs of \RefTheorem{thm:length-smallest-ms} shows
    that a significant part of the maximal segments (at least
    $\Omega(1/(\log m))$) has an average length that grows no faster
    than $\Omega(m^{1/3} \log m)$. This fact is confirmed with
    experiments. \RefFig{fig:ms-curv}, left, plots the size of maximal
    segments for a disk digitized with increasing
    resolution. The average size is closer to $m^{1/3}$ than
    to $\sqrt{m}$.
    
    \begin{figure}
      \begin{center}
        \epsfig{file=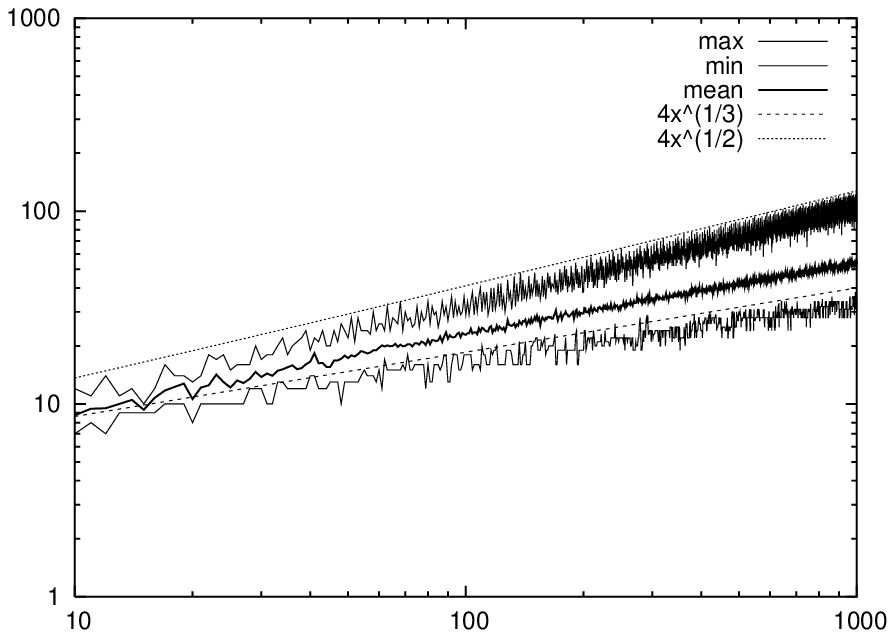,width=0.45\textwidth}
        \epsfig{file=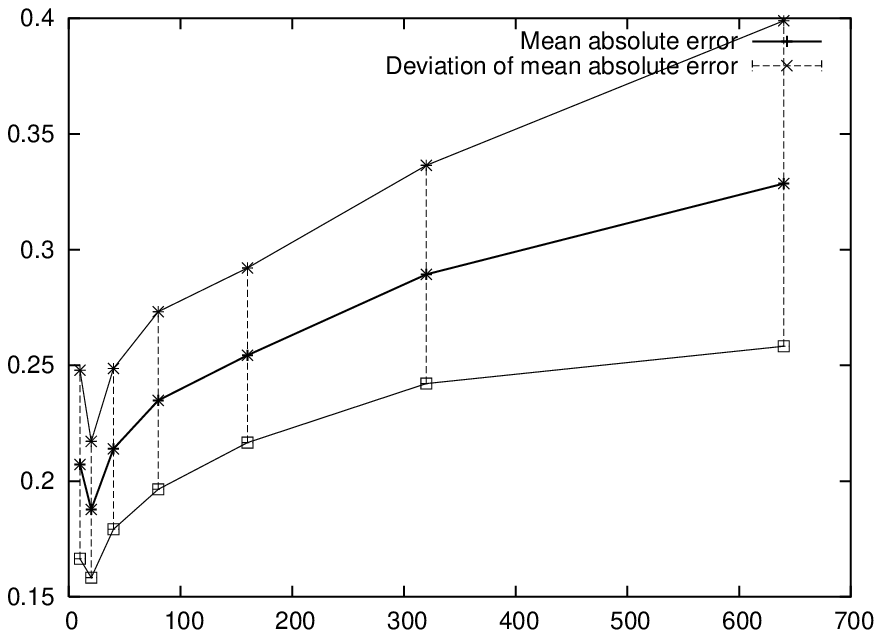,width=0.45\textwidth}
      \end{center}
      \caption{For both curves, the digitized shape is a disk of radius 1
        and the abscissa is the digitization resolution. Left: plot in
        log-space of the $\mathcal{L}^1$-size of maximal segments. Right:
        plot of the mean and standard deviation of the absolute error of
        curvature estimation, $|\hat{\kappa} - 1|$ (expected curvature is
        1).
        \label{fig:ms-curv}}
    \end{figure}

    \subsection{Asymptotic convergence of discrete geometric estimators}
    A useful property that a discrete geometric estimator may have is to
    converge toward the geometric quantity of the continuous shape
    boundary when the digitization grid gets finer
    \cite{Coeurjolly02,Coeurjolly03,Klette00}.  It may be expressed as
    follows,
    %It is expressed for tangent estimation as follows:
    \begin{definition}
      Let $\mathcal{F}$ be any geometric descriptor on the shape $S$
      with boundary $\gamma$ and digitizations $\Gamma(m)$. The
      discrete geometric estimator $\mathcal{E}$ {\em asymptotically
        converges} toward the descriptor $\mathcal{F}$ for $\gamma$
      iff
      \begin{equation}
        \left| \mathcal{E}(\Gamma(m)) - \mathcal{F}(\gamma) \right|
        \le \epsilon(m) \mathrm{~with~} \lim_{m \rightarrow +\infty} \epsilon(m) =
        0.
      \end{equation}
    \end{definition}
%%     Of course, interesting discrete geometric estimator should
%%     converge for a large class of curves. 
    We now recall the definition of a discrete curvature estimator
    based on DSS recognition \cite{Coeurjolly02}.
    \begin{definition}
      \label{def:curvature-by-circumcircle}
      Let $P$ be any point on a discrete contour, $Q=B(P)$ and
      $R=F(P)$ are the extremities of the longest DSS starting from
      $P$ (called {\em half-tangents}). Then the {\em curvature
      estimator by circumcircle} $\hat{\kappa}(P)$ is the inverse of
      the radius of the circle circumscribed to $P$, $Q$ and $R$,
      rescaled by the resolution $m$.
    \end{definition}
    Experiments show that this estimator rather correctly estimates
    the curvature of discrete circles {\em on average} ($\approx 10$\%
    error). It is indeed better than any other curvature estimators
    proposed in the litterature.  Theorem~B.4 of \cite{Coeurjolly02}
    demonstrates the {\em asymptotic convergence} of this curvature
    estimator, subject to the hypothesis:
    \begin{hypo}
      \label{hyp:coeurjolly}
      Half-tangents on digitized boundaries grow at a rate of
      $\Omega(\sqrt{m})$ with the resolution $m$.
    \end{hypo}
    However, with our study of maximal segments, we can state that
    
    \begin{claim}
      Hypothesis~\ref{hyp:coeurjolly} is {\em not} verified for
      digitizations of $C^3$-curves with strictly positive curvature.
      We cannot conclude on the asymptotic convergence of the
      curvature estimator by circumcircle.
%%       Theorem~B.4 \cite{Coeurjolly02} proves the asymptotic
%%       convergence of the curvature estimator by osculating circle with
%%       the hypothesis: ``half-tangents grow at a rate of
%%       $\Omega(\sqrt{m})$ as the resolution $m$ increases'', but this
%%       hypothesis is {\em not} verified for $C^3$-curves with strictly
%%       positive curvature.
    \end{claim}
    \begin{proof}
      It is enough to note that half-tangents, being DSS, are included in
      maximal segments and may not be longer. Furthermore, since maximal
      segments cover the whole digital contour, some half-tangents will be
      included in the smallest maximal segments. Since the smallest
      maximal segments are no longer than $\Omega(m^{1/3} \log
      m)$ (\RefTheorem{thm:asymptotic-size-ms}), the length of some
      half-tangents has the same upper bound, which is smaller than
      $\Omega(\sqrt{m})$.
      \qed
    \end{proof}   
    
    The asymptotic convergence of a curvature estimator is thus still
    an open problem. Furthermore, precise experimental evaluation of
    this estimator indicates that it is most certainly not
    asymptotically convergent, although it is actually on average one
    of the most stable discrete curvature estimator (see
    \RefFig{fig:ms-curv}, right). Former experimental evaluations of this
    estimator were averaging the curvature estimates on all contour
    points. The convergence of the average of all curvatures does not
    induce the convergence of the curvature at one point.

    %------------------------------------------------------------------------------
    \section{Conclusion}
    We show in this paper the relations between edges of convex hulls
    and maximal segments in terms of number and sizes. We provide an
    asymptotical analysis of the worst cases of both measures. A
    consequence of the study is the refutation of an hypothesis
    related to the asymptotic growth of maximal segments and which was
    essential in proving the convergence of a curvature
    estimator based on DSS and circumcircles
    \cite{Coeurjolly02}. Our work also applied to digital tangents
    since their convergence relies on the same hypothesis.  The existence of
    a convergent discrete estimator of curvature based on DSS is thus
    still a challenging problem and we are currently investigating it.

%------------------------------------------------------------------------------
    \bibliographystyle{plain}
    \bibliography{biblio}

%------------------------------------------------------------------------------
    \appendix 

    \section{Proof of \RefProposition{prop:even-pattern}}

    \noindent {\bf Proposition~\ref{prop:even-pattern}:}
    {\em A pattern with an even complexity (say $n=2i$) is such that
    $\mathbf{U_{1}L_{1}}=\VEC{q_{2i-1} + 1}{p_{2i-1} -1 }$ and
    $\mathbf{L_{1}U_{2}}= (u_{2i}-1)\VEC{q_{2i-1}}{p_{2i-1}} + \VEC{
      q_{2i-2}}{p_{2i-2}} + \VEC{-1}{1}$.  Moreover the DSS
    $[U_{1},L_{1}]$ has $E(z_{2i-2})^{u_{2i-1}}$ as a left factor, and
    the DSS $[L_{1}, U_{2}]$ has $E(z_{2i-1})^{u_{2i}-1}$ as a right
    factor.}
    \smallskip 

    \begin{proof} From \Eq{pattern:rec:dif} we have:
    $p_{2i}q_{2i-1} - p_{2i-1}q_{2i} = (-1)^{2i+1} = -1$, which can be
    rewritten as: $a ( -q_{2i-1}) - b (-p_{2i-1}) =1$ and eventually
    $a ( q_{2i} -q_{2i-1}) - b (p_{2i} -p_{2i-1}) =1$. We clearly
    obtain the Bézout coefficients. From its remainder we get the
    relatives coordinates of $L_{1}$, as: $\mathbf{U_{1}L_{1}}=
    \VEC{q_{2i-1}+1}{p_{2i-1} -1 }$.  From $\mathbf{L_{1}U_{2}} =
    -\mathbf{U_{1}L_{1}} + \mathbf{U_{1}U_{2}}$ we get :
    $\mathbf{L_{1}U_{2}} = \VEC{(u_{2i} - 1)q_{2i-1} + q_{2i-2} -
    1}{(u_{2i} - 1) p_{2i-1} + p_{2i-2} + 1}$.  Using
    $E(z_{2i})=E(z_{2i-2})^{u_{2i-1}+1}E(z_{2i-3})E(z_{2i-1})^{u_{2i}
    - 1}$ and $\mathbf{U_{1}L_{1}} = \VEC{u_{2i-1}q_{2i-2}+q_{2i-3} +
    1}{u_{2i-1}p_{2i-2}+p_{2i-3} - 1}$, it is clear that
    $E(z_{2i-2})^{u_{2i-1}}$ is a left factor of the DSS
    $[U_{1}$$L_{1}]$. From \Eq{pattern:rec:even} and
    $\mathbf{L_{1}U_{2}}$ we clearly see that $E(z_{2i-1})^{u_{2i}-1}$
    is a right factor of the DSS $[L_{1}U_{2}]$. \qed
    \end{proof}

%%     \begin{figure}[htbp]
%%       \begin{center}
%%         \input{IllustrationMSEdges.pstex_t} 
%%         \caption{We here illustrate the construction of the edge
%%           $e_{-2}$ and $e_{-1}$ from $e_{0}$.}
%%         \label{fig:edgeIllustration}
%%       \end{center}
%%     \end{figure}

    \section{Proof of \RefTheorem{thm:msMaxedge}}
    
    \begin{lemma}    
      \label{lem:patternRL}
      We call $P_{n}$ a pattern of complexity $n$ whose Freeman code
      is $E(z_{n})$. One can build strict right and left factors
      (called respectively $R$ and $L$) of $P_{n}$ such that:
      \begin{itemize}
	\item[(i)] $[RP_{n}]$, $[P_{n}L]$ and $[RP_{n}L]$ are DSS of
	  slope $z_{n}$,
	\item [(ii)] $R$ and $L$ are patterns (or successions of the
	  same pattern) ,
	\item [(iii)] $RP_{n}$, $P_{n}L$ and $RP_{n}L$ are not
	  patterns,
	\item [(iv)] the slope of $R$ is greater than that of $P_{n}$
	  and the slope of $P_{n}$ is greater than that of $L$,

	  %%  $[RP_n]$, $[P_{n}L]$ and $[RP_{n}L]$ are digital convex
	  %%  regions,
	\item [(v)] maximal complexity of slope of $R$ and $L$ depends
	  on parity of $n$: \\
	  \begin{tabular}{|c|c|c|}
	    \hline 
	    Complexity of $P_{n}$ & maximal complexity of $R$ & maximal complexity of $L$ \\
	    \hline
	    $2i+1$                & $2i+1$                    & $2i$                      \\
	    \hline
	    $2i$                  & $2i-1$                    & $2i$                      \\
	    \hline
	  \end{tabular}
	  
	  \item [(vi)] Complexity of factors obtained by substracting
	    $R$ or $L$ from $P_{n}$ depends on parity of $n$: \\
	  \begin{tabular}{|c|c|c|}
	    \hline 
	    Complexity of $P_{n}$ & complexity of $P_{n} \smallsetminus R$ & complexity of $P_{n} \smallsetminus L$ \\
	    \hline
	    $2i+1$                & $2i$                                   & $2i+1$                                 \\
	    \hline 
	    $2i$                  & $2i$                                   & $2i-1$                                 \\
	    \hline
	  \end{tabular}
	  
      \end{itemize}
    \end{lemma}  
    \begin{proof}
      Since $R$ and $L$ are strict factors of $P_{n}$, their Freeman
      moves are compatible with those of $E(z_{n})$, giving same slope
      when $R$,$P_{n}$ and $L$ are put together. Thus $[RP_{n}L]$ is a
      DSS of slope $z_{n}$. From digital straightness we clearly have
      digital convexity (see \cite{KIM:PAMI:1982}). Upper leaning
      points of this DSS are located at extremities of $P_{n}$.
      
      We simply choose among strict factors $R$ and $L$ those that are
      patterns so that they fit descriptions given in
      \Eq{pattern:rec:odd} and \Eq{pattern:rec:even}. We may now
      describe them given the parity of $n$.
      
      Consider the case where $n$ is odd (say $n=2i+1$), from
      \Eq{pattern:rec:odd} we get: $R = E(z_{2i})^{u_{2i+1} -
      r}E(z_{2i-1})$ and $L= E(z_{2i})^{u_{2i+1} - l}$ with $r>0$ and
      $l>0$. If $R$ and $L$ are longer patterns, they are not anymore
      strict factors of $P_{2i+1}$.  We see that $R$ is a pattern of
      complexity $2i+1$ and that $L$ is a succession of the pattern
      $E(z_{2i})$, with a complexity of $2i$.  The slope of $R$ equals
      $z'_{2i+1}=[0,u_{1},\ldots,u_{2i}, u_{2i+1}-r] =
      \frac{p'_{2i+1}}{q'_{2i+1}}$. From \Eq{pattern:rec} we get that
      $\frac{p_{2i+1}}{q_{2i+1}} = \frac{p'_{2i+1} +
      rp_{2i}}{q'_{2i+1}+ rq_{2i}}$. The sign of $z'_{2i+1} -
      z_{2i+1}$ is that of $p'_{2i+1}q_{2i}-q'_{2i+1}p_{2i}$, and is
      positive (see \Eq{pattern:rec:dif}). Thus the slope of $R$ is
      greater than that of $P_{2i+1}$. Same reasoning applied to
      $z_{2i+1} - z_{2i}$ brings that the slope of $P_{2i+1}$ is
      greater than that of $L$. Factor obtained by substracting $R$
      from $P_{2i+1}$ equals $E(z_{2i})^{r}$ and substracting $L$ from
      $P_{2i+1}$ gives $E(z_{2i})^{l}E(z_{2i-1})$.
      %% The swlope of $L$ is $z_{2i}$ and ...
      
      Consider now that $n$ is even (say $n=2i$), from
      \Eq{pattern:rec:even} we get: $R=E(z_{2i-1})^{u_{2i} - r}$ and
      $L=E(z_{2i-2})E(z_{2i-1})^{u_{2i} - l}$. If $R$ and $L$ are
      longer patterns, they are not anymore strict factors of
      $P_{2i}$. Clearly, $R$ has a complexity of $2i-1$ and that of
      $L$ equals $2i$.  The slope of $L$ equals $z'_{2i} =
      [0,u_{1},\ldots,u_{2i-1}, u_{2i}-l] = \frac{p'_{2i}}{q'_{2i}}$.
      From \Eq{pattern:rec} we get that $\frac{p_{2i}}{q_{2i}} =
      \frac{p'_{2i} + lp_{2i-1}}{q'_{2i}+ lq_{2i-1}}$. The sign of
      $z_{2i} - z'_{2i}$ is that of $q'_{2i}p_{2i-1}-p'_{2i}q_{2i-1}$,
      and is positive (see \Eq{pattern:rec:dif}). Thus the slope of
      $P_{n}$ is greater than that of $L$. Same reasoning applied to
      $z_{2i-1} - z_{2i}$ brings that the slope of $R$ is greater than
      that of $P_{n}$.  Factor obtained by substracting $R$ from
      $P_{2i}$ equals $E(z_{2i-2})E(z_{2i-1})^{r}$ and substracting
      $L$ from $P_{2i}$ gives $E(z_{2i-1})^{l}$.
      %% The swlope of $R$ is exactly $z_{2i-1}$ and ...

      It is now clear that slopes are strictly decreasing from $R$ to
      $P_{n}$ and from $P_{n}$ to $L$ whatever the parity of $n$. \qed
      
    \end{proof}
       
    \begin{theorem} 
      Let $E$ be a supporting edge whose slope has a complexity $n$,
      $n \ge 2$, then the maximal segment containing $E$ includes at
      most $n$ other edges on each side of $E$.
      %%       Let $E$ be a supporting edge of a CDP whose slope $\frac{a}{b}$
      %%       equals $[u_0, \ldots, u_i, \ldots, u_n]$ with $n \geq 2$ and
      %%       containing one or more pattern(s) of characteristics $(a,b)$, any
      %%       maximal segment containing $E$ may contains at most $n$ edges on
      %%       each side of $E$.
    \end{theorem}
    
    \begin{proof}
      We construct $2n$ digital edges around $E$:
      \begin{itemize}
      \item $(R_{i})_{1 \leq i \leq n}$ at left of $E$, 
      \item $(L_{i})_{1 \leq i \leq n}$ at right of $E$.
      \end{itemize}
      
      These edges are such that $[R_{n} \ldots R_{i} \ldots R_{1} E
	L_{1} \ldots L_{j} \ldots L_{n}]$ is a DSS of slope $z_{n} =
      a/b$ and has no other upper leaning points but those located on
      $E$. $E$ may contain several times the pattern $E(z_{n})$. It
      is clear that $R_{n} \ldots R_{i} \ldots R_{1}$ (resp. $L_{1}
      \ldots L_{j} \ldots L_{n}$) has to be a right (resp. left)
      strict factor of $E(z_{n})$. Moreover $R_{i}$ is a right
      strict factor of $E(z_{n}) \smallsetminus R_{1}\ldots R_{i-1}$
	and $L_{i}$ is a left strict factor of $E(z_{n})
	\smallsetminus L_{1}\ldots L_{i-1}$. From
	\RefProposition{prop:edge-pattern} any of the digital edges
	$(R_{i})_{1 \leq i \leq n}$ and $(L_{i})_{1 \leq i \leq n}$ is
	a pattern or a succession of the same pattern.  From
	\Eq{pattern:rec:odd} and \Eq{pattern:rec:even} two successive
	digital edges with same complexity (say $n$) cannot form a
	right or left strict factor of a pattern with same
	complexity. Thus complexities of $(R_{i})_{1 \leq i \leq n}$
	and $(L_{i})_{1 \leq i \leq n}$ are decreasing when $i$
	increases. Moreover to fullfil convexity properties, slopes of
	edges are decreasing from $R_{n}$ to $L_{n}$.
	
	We now build $(R_{i})_{1 \leq i \leq n}$ when n is odd (say
	$n=2i+1$). From \RefLemma{lem:patternRL}, $R_{1}$ has a
	complexity that equals $2i+1$ and $R_{2}$ is a right strict
	factor of a pattern whose complexity equals $2i$. Applying
	\RefLemma{lem:patternRL} brings $R_{2}$ with a complexity of
	$2i-1$. Applying the same reasoning recursively brings other
	edges as shown on Table~\ref{tab:construction-odd}. \RefLemma{lem:patternRL}
	also implies decreasing slopes and give upper bounds in
	complexity of factors.

	Constructions for the three other cases are given in Tables
	~\ref{tab:construction-odd} and ~\ref{tab:construction-even}
	and follow the same reasoning. To satisfy full decomposition
	each $(u_{k})_{1 \leq n}$ has to be equal or greater than
	$2$. If this condition is not meet for some $k$, than steps
	associated with it (e.g. any factors containing $u_{k}-r_{j}$
	or $u_{k}-l_{j}$ as powers of some pattern) are skipped.  This
	concludes the proof. \qed
    \end{proof}
    
    \begin{table}[ht]
      \begin{center}\leavevmode
	\caption{Constructions of $(R_{i})_{1 \leq i \leq n}$ and $(L_{i})_{1 \leq i \leq n}$ given $n$ odd . \label{tab:construction-odd}}
	\begin{tabular}{c}
	  
	  Constructions of $(R_{i})_{1 \leq i \leq n}$ when $n=2i+1$ \\
	  \begin{minipage}[c]{\textwidth}
	     \begin{center}\leavevmode
	       %	       \caption{Constructions of $(R_{i})_{1 \leq i \leq n}$ when
	       %		 $n=2i+1$. \label{tab:oddR}} 
	       %----------
	       $ \begin{array}{|c|c|c|}
		 \hline
		 \textrm{Factor}                   &  \textrm{Freeman moves}                     & \textrm{Complexity} \\
		 \hline
		 R_{1}                             &  E(z_{2i})^{u_{2i+1} - r_{1} }E(z_{2i-1})   & 2i+1  \\
		 \hline
		 R_{2}                             &  E(z_{2i-1})^{u_{2i} - r_{2}}               & 2i-1  \\
		 \hline
		 R_{3}                             &  E(z_{2i-2})^{u_{2i-1} - r_{3}}E(z_{2i-3})  & 2i-1  \\ 
		 \hline
		 R_{4}                             &  E(z_{2i-3})^{u_{2i-2} - r_{4}}             & 2i-3  \\
		 \hline
		 \vdots                            &      \vdots                                 & \vdots\\
		 \hline
		 R_{2j}                            &  E(z_{2i+1 - 2j})^{u_{2i+2-2j} - r_{2j}}    & 2i+1- 2j \\
		 \hline
		 R_{2j+1}                          &  E(z_{2i- 2j})^{u_{2i+1 - 2j} - r_{2j+1}} E(z_{2i -1 - 2j }) & 2i+1 -2j\\
		 \hline
		 \vdots                            &      \vdots                                 & \vdots\\
		 \hline
		 R_{2i+1}                          &  0^{u_{1}-r_{2i+1}}1                        &  1 \\
		 \hline
	       \end{array} $
	       %----------
	     \end{center}
	  \end{minipage}
	  \vspace {0.3cm}
	  \\
	  Constructions of $(L_{i})_{1 \leq i \leq n}$ when $n=2i+1$ \\	  
	  \begin{minipage}[c]{\textwidth}
	    \begin{center}\leavevmode
	      %	\caption{Constructions of $(L_{i})_{1 \leq i \leq n}$ when
	      %        $n=2i+1$. \label{tab:oddL}} 
	      %----------
	      $ \begin{array}{|c|c|c|}
		\hline
		\textrm{Factor}                   &  \textrm{Freeman moves}                     & \textrm{Complexity} \\
		\hline
		L_{1}                             &  E(z_{2i})^{u_{2i+1} - l_{1} }              & 2i \\
		\hline
		L_{2}                             &  E(z_{2i-2})E(z_{2i-1})^{u_{2i} - l_{2}}    & 2i \\
		\hline
		L_{3}                             &  E(z_{2i-2})^{u_{2i-1} - l_{3}}             & 2i-2 \\
		\hline
		L_{4}                             &  E(z_{2i-4}) E(z_{2i-3})^{u_{2i-2} - l_{4}} & 2i-2 \\ 
		\hline
		\vdots                            &      \vdots                                 & \vdots \\
		\hline
		L_{2j}                            &  E(z_{2i - 2j})E(z_{2i+1-2j})^{u_{2i+2-2j} - l_{2j}}   & 2i+2-2j \\
		\hline
		L_{2j+1}                          &  E(z_{2i- 2j})^{u_{2i+1 - 2j} - l_{2j+1}}              & 2i-2j \\
		\hline
		\vdots                            &      \vdots                                 & \vdots \\
		\hline
		L_{2i+1}                          &  0^{u_{1}-l_{2i+1}}                         &  0 \\
		\hline
	      \end{array} $
	      %----------
	    \end{center}
	  \end{minipage}
	\end{tabular}
      \end{center}
    \end{table}
	  
    \begin{table}[ht]
      \begin{center}\leavevmode
	\caption{Constructions of $(R_{i})_{1 \leq i \leq n}$ and $(L_{i})_{1 \leq i \leq n}$ given $n$ even. \label{tab:construction-even}}
	\begin{tabular}{c}
	  Constructions of $(R_{i})_{1 \leq i \leq n}$ when  $n=2i$ \\
	  \begin{minipage}[c]{\textwidth}
	    \begin{center}\leavevmode
	      
	      %	      \caption{Constructions of $(R_{i})_{1 \leq i \leq n}$ when
	      %        $n=2i$. \label{tab:evenR}}
	      %----------
	      $ \begin{array}{|c|c|c|}
		\hline
		\textrm{Factor}                   &  \textrm{Freeman moves}                     & \textrm{Complexity} \\
		\hline
		R_{1}                             &  E(z_{2i-1})^{u_{2i} - r_{1}}               & 2i-1 \\
		\hline
		R_{2}                             &  E(z_{2i-2})^{u_{2i-1} - r_{2}}E(z_{2i-3})  & 2i-1 \\ 
		\hline
		R_{3}                             &  E(z_{2i-3})^{u_{2i-2} - r_{3}}             & 2i-3 \\
		\hline
		R_{4}                             &  E(z_{2i-4})^{u_{2i-3} - r_{4} }E(z_{2i-5}) & 2i-3 \\
		\hline
		\vdots                            &      \vdots                                 & \vdots \\
		\hline
		R_{2j}                            &  E(z_{2i- 2j})^{u_{2i+1 - 2j} - r_{2j}} E(z_{2i -1 - 2j }) & 2i+1-2j \\
		\hline
		R_{2j+1}                          &  E(z_{2i -1 - 2j})^{u_{2i-2j} - r_{2j+1}}   & 2i-1 -2j \\
		\hline
		\vdots                            &      \vdots                                 & \vdots \\
		\hline
		R_{2i}                            &  0^{u_{1}-r_{2i}}1                          & 1  \\
		\hline
	      \end{array} $
	      %----------
	    \end{center}
	  \end{minipage}
	  \vspace {0.3cm}
	  \\
	  Constructions of $(L_{i})_{1 \leq i \leq n}$ when  $n=2i$ \\	  	    
	  \begin{minipage}[c]{\textwidth}
	    \begin{center}\leavevmode
		%	      \caption{Constructions of $(L_{i})_{1 \leq i \leq n}$ when
	      %		$n=2i$. \label{tab:evenL}}
	      %----------
	      $ \begin{array}{|c|c|c|}
		\hline
		\textrm{Factor}                   &  \textrm{Freeman moves}                     & \textrm{Complexity} \\
		\hline
		L_{1}                             &  E(z_{2i-2})E(z_{2i-1})^{u_{2i} - l_{1}}    & 2i \\
		\hline
		L_{2}                             &  E(z_{2i-2})^{u_{2i-1} - l_{2}}             & 2i-2 \\
		\hline
		L_{3}                             &  E(z_{2i-4}) E(z_{2i-3})^{u_{2i-2} - l_{3}} & 2i-2 \\ 
		\hline
		L_{4}                             &  E(z_{2i-4})^{u_{2i-3} - l_{4} }            & 2i-4   \\
		\hline
		\vdots                            &      \vdots                                 & \vdots \\
		\hline
		L_{2j}                            &  E(z_{2i- 2j})^{u_{2i+1 - 2j} - l_{2j}}     & 2i-2j \\
		\hline
		L_{2j+1}                          &  E(z_{2i -2 - 2j})E(z_{2i-1-2j})^{u_{2i-2j} - l_{2j+1}}   & 2i-2j \\
		\hline
		\vdots                            &      \vdots                                 & \vdots \\
		\hline
		L_{2i}                            &  0^{u_{1}-l_{2i}}                           & 0 \\
		\hline
	      \end{array} $
	      %----------
	    \end{center}
	  \end{minipage}
	  \\
	\end{tabular}
      \end{center}
    \end{table}

\end{document}